\documentclass{amsart}

\usepackage{lineno,hyperref}
\modulolinenumbers[5]
 \usepackage[foot]{amsaddr}

\usepackage{amssymb}
\usepackage{mathtools}
\mathtoolsset{showonlyrefs}
\usepackage{xparse}
\usepackage{subcaption}
\usepackage{amsthm} % before newtxmath because of \openbox
\usepackage{thmtools}
\usepackage{thm-restate}
\usepackage[utf8]{inputenc}
\usepackage[T1]{fontenc}
\usepackage{newtxtext}
\usepackage[vvarbb]{newtxmath}
\usepackage{adjustbox}

\usepackage{booktabs}       % bigprofessional-quality tables
\usepackage{microtype}      % microtypography
\usepackage{xspace}
\usepackage{paralist}
\usepackage{enumerate}

\usepackage{xcolor}
\usepackage{graphicx}

\definecolor{darkblue}{RGB}{0, 0, 102}
\definecolor{IKB}{RGB}{0, 47, 167}

\newtheorem{theorem}{Theorem}[section]
\newtheorem{maintheorem}{Main Theorem}[section]
\newtheorem{proposition}[theorem]{Proposition}
\newtheorem{corollary}[theorem]{Corollary}
\newtheorem{lemma}[theorem]{Lemma}

\newtheorem{definition}[theorem]{Definition}
\newtheorem{remark}[theorem]{Remark}

\newtheorem{assumption}[theorem]{Assumption}

\newcommand{\D}{\ensuremath{\mathcal{D}}}

\newcommand{\cL}{\ensuremath{\mathcal{L}}}

\newcommand{\E}{\begin{equation}}
\newcommand{\EE}{\end{equation}}

\newcommand{\R}{\mathbb{R}}
\newcommand{\NN}{\mathbb{N}}

\newcommand{\proj}{\mathop{\rm proj}}

\newcommand{\argmin}{\mathop{\rm argmin}}
\newcommand{\conv}{\mathop{\rm conv}}

\DeclareMathOperator{\GRM}{GRM}
\DeclareMathOperator{\ERM}{ERM}

\makeatletter
\def\moverlay{\mathpalette\mov@rlay}
\def\mov@rlay#1#2{\leavevmode\vtop{%
   \baselineskip\z@skip \lineskiplimit-\maxdimen
   \ialign{\hfil$\m@th#1##$\hfil\cr#2\crcr}}}
\newcommand{\charfusion}[3][\mathord]{
    #1{\ifx#1\mathop\vphantom{#2}\fi
        \mathpalette\mov@rlay{#2\cr#3}
      }
    \ifx#1\mathop\expandafter\displaylimits\fi}
\makeatother

\providecommand{\bigcupdot}{\charfusion[\mathop]{\bigcup}{\cdot}}

\NewDocumentCommand{\probability}{d()om}{%
  \operatorname{\mathbb{P}}%
  \IfValueT{#1}{\sb{#1}}%
  \ensuremath{\left[#3%
    \IfValueT{#2}{\,\middle|\,#2}\right]}}
\NewDocumentCommand{\expectation}{d()om}%
  {\ensuremath{\operatorname{\mathbb{E}}%
      \IfValueT{#1}{\sb{#1}}%
      \left[#3%
    \IfValueT{#2}{\,\middle|\,#2}\right]}}

\NewDocumentCommand{\variance}{d()om}%
  {\ensuremath{\operatorname{\mathbb{V}}%
      \IfValueT{#1}{\sb{#1}}%
      \left[#3%
    \IfValueT{#2}{\,\middle|\,#2}\right]}}

\DeclarePairedDelimiter{\card}{\lvert}{\rvert}

\usepackage{tikz}
\usetikzlibrary{arrows,arrows.meta,intersections,backgrounds,shapes.geometric}
\usepackage{todonotes}

\usepackage{amsfonts}
\usepackage{epstopdf}

\begin{document}

\title{Principled Deep Neural Network Training through Linear Programming}

\author{Daniel Bienstock}
\address[DB]{Industrial Engineering and Operations Research, Columbia University, USA}
\email{dano@columbia.edu}
\author{Gonzalo Mu\~noz}
\address[GM]{Institute of Engineering Sciences, Universidad de O'Higgins, Chile} 
\email{gonzalo.munoz@uoh.cl}
\author{Sebastian Pokutta}
\address[SP]{Zuse Institute Berlin, Germany}
 \email[SP]{pokutta@zib.de}

\begin{abstract}
  Deep learning has received much attention lately due to the impressive empirical performance achieved by training algorithms. Consequently, a need for a better theoretical understanding of these problems has become more evident in recent years. In this work, using a unified framework, we show that there exists a polyhedron which encodes simultaneously all possible deep neural network training problems that can arise from a given architecture, activation functions, loss function, and sample-size. Notably, the size of the polyhedral representation depends only linearly on the sample-size, and a better dependency on several other network parameters is unlikely (assuming $P\neq NP$). Additionally, we use our polyhedral representation to obtain new and better computational complexity results for training problems of well-known neural network architectures. Our results provide a new perspective on training problems through the lens of polyhedral theory and reveal a strong structure arising from these problems.
\end{abstract}


\maketitle

\section{Introduction}
\label{sec:introduction}
Deep Learning is a powerful tool for modeling complex learning tasks. Its
versatility allows for nuanced architectures that capture
various setups of interest and has demonstrated a nearly unrivaled
performance on learning tasks across many domains. 
This has recently triggered a
significant interest in the theoretical analysis of training
such networks. The training problem is usually formulated as an
\emph{empirical risk minimization problem (ERM)} that can be phrased as
\begin{equation}\label{eq:erm}
  \min_{\phi \in \Phi} \frac{1}{D} \sum_{i=1}^D
  \ell(f(\hat{x}^i,\phi), \hat{y}^i),
\end{equation}
where \(\ell\) is some \emph{loss function},
$(\hat{x}^i, \hat{y}^i)_{i=1}^D$ is an i.i.d.~sample from some data
distribution \(\D\), and \(f\) is a neural network architecture
parameterized by \(\phi \in \Phi\) with \(\Phi\) being the parameter
space of the considered architecture (e.g., network weights). The empirical risk minimization
problem is solved \emph{in lieu} of the \emph{general risk
  minimization problem (GRM)}
\[\min_{\phi \in \Phi} \expectation((x,y) \in \D) {\ell(f(x,\phi),
  y)}\] 
 which is usually impossible to solve due to the
inaccessibility of \(\D\). 

While most efforts on handling \eqref{eq:erm} have been aimed at practical performance, much less research has been conducted in understanding its theoretical difficulty from an optimization standpoint. In particular, only few results account for the effect of $D$, the sample size, in the structure and hardness of \eqref{eq:erm}. In this work, we contribute to the understanding of this problem by showing there exists
a \emph{polyhedral encoding} of empirical risk minimization problems in
\eqref{eq:erm} associated with the learning problems for various
architectures with remarkable features.
For a given architecture and sample size, our polyhedron encodes
approximately \emph{all} possible empirical risk minimization problems
with that sample size simultaneously. The size of the polyhedron is
roughly (singly!) exponential in 
the input dimension and in the parameter space dimension, but, notably, \emph{linear} in the size of the sample. This result provides a new perspective on training problems and also yields new bounds on the computational complexity of various training problems from a unified approach. 

Throughout this work we assume both data and
parameters to be well-scaled, which is a common assumption and mainly
serves to simplify the representation of our results; the main assumption is the reasonable boundedness, which can be assumed without significant loss of generality as actual computations assume boundedness in any case (see also \cite{liao2018surprising} for arguments advocating the use of normalized coefficients in neural networks). More specifically, we assume $\Phi \subseteq [-1,1]^N$ as well as $(x,y) \sim \mathcal D$ satisfies $(x,y) \in [-1,1]^n \times [-1,1]^m$.

\subsection{Related Work}
\label{sec:related-work}

We are not aware of any encoding representing multiple training problems simultaneously. However, given its implications on the training problems for a fixed sample, our work is related to \cite{zhang2016l1}, \cite{goel2017reliably},
and \cite{arora1611understanding}. In
\cite{zhang2016l1} the authors show that \(\ell_1\)-regularized
networks can be learned \emph{improperly}\footnote{In \emph{improper} learning the predictor may not be a neural network, but will behave similarly to one.} in polynomial time 
(with an exponential architecture-dependent
constant) for networks with ReLU-like activations. 
These results were generalized by \cite{goel2017reliably} to ReLU activations, but the running time obtained is not polynomial. In contrast, \cite{arora1611understanding}
considered \emph{exact} learning however only for one hidden layer.

To the best of our knowledge, the only work where a polyhedral approach is used to analyze the computational complexity of training of neural networks is \cite{arora1611understanding}, where the authors solve \eqref{eq:erm} for 1 hidden layer using a collection of convex optimization problems over a polyhedral feasible region. 
In practice, even though the most common methods used for tackling \eqref{eq:erm} are based on Stochastic Gradient Descent (SGD), there are some notable and surprising examples where linear programming has been used to train neural networks. For example, in \cite{bennett1990neural,bennett1992robust,roy1993polynomial,mukhopadhyay1993polynomial}
the authors construct a 1-hidden layer network by sequentially increasing the number of nodes of the hidden layer and solving linear programs to update the weights, until a certain target loss is met.

Linear programming tools have also been used within SGD-type methods in order to compute optimal \emph{step-sizes} in the optimization of \eqref{eq:erm} \cite{berrada2018deep} or to strictly enforce structure in $\Phi$ using a Frank-Wolfe approach instead of SGD \cite{pokutta2020deep,xie2020efficient} . Finally, a back-propagation-like algorithm for training neural network, which solves Mixed-Integer Linear problems in each layer, was recently proposed as an alternative to SGD \cite{goebbelstraining2021}.

Other notable uses of Mixed-Integer and Linear Programming technology in other aspects
of Deep Learning are include feature visualization \cite{fischetti2018deep}, generating adversarial examples \cite{cheng2017maximum, fischetti2018deep, khalil2018combinatorial}, 
counting linear regions of a Deep Neural Network \cite{serra2017bounding}, performing inference \cite{amos2017input} and providing strong convex relaxations for trained neural networks \cite{anderson2018strong}.

We refer the reader to the book by \cite{goodfellow2016deep} and the surveys by \cite{curtis2017optimization,bottou2018optimization,wright2018optimization} for in-depth descriptions and analyses of the most commonly used training neural networks.

\subsection{Contribution}
\label{sec:contribution}

In this work, we consider neural networks with an arbitrary number of layers \(k\) and a wide range of
activations, loss functions, and architectures. 
We first establish a
general framework that yields a polyhedral representation of generic (regularized) ERM problems. Our approach is motivated by the work of \cite{bienstock2018lp} which describes schemes for approximate reformulation of many non-convex optimization problems as linear programs.
Our results allow the encoding and analysis of various deep network setups simply by plugging-in complexity measures for the constituting elements such
as layer architecture, activation functions, and loss functions. 

\subsubsection{Polyhedral encoding of ERM problems} 
Given $\epsilon > 0$ and a sample size
$D$ there exists a \emph{data-independent} polytope (it can be written down \emph{before} seeing the data) with
the following properties:

\paragraph{Solving the ERM problem to \(\epsilon\)-optimality in data-dependent faces}
  For every realized sample $(\hat X, \hat Y) = (\hat{x}^i,\hat{y}^i)_{i=1}^D$
  there is a \emph{face} $\mathcal{F}_{\hat{X},\hat{Y}}\subseteq P$ of said polytope 
  such that optimizing certain \emph{linear} function over
  $\mathcal{F}_{\hat{X},\hat{Y}}$ solves \eqref{eq:erm} to
  \(\epsilon\)-optimality returning a parametrization
  \(\tilde \phi \in \Phi\) which is part of our hypothesis class.
As such, the polytope has a \emph{build-once-solve-many} feature. 

\paragraph{Size of the polytope}
The \emph{size}, measured as bit complexity, of the polytope is roughly
\(
O( \left(2\mathcal{L}/\epsilon\right)^{N+n+m}
  D)\) where $\mathcal{L}$ is a constant depending on $\ell $,
$f$, and $\Phi$ that we will introduce later, $n,m$ are the dimensions
of the data points, i.e., $\hat{x}^i \in \mathbb{R}^n$ and $\hat{y}^i \in \mathbb{R}^m$
for all $i\in [D]$, and $N$ is the dimension of the parameter space $\Phi$. 

It is important to mention that $\mathcal{L}$ measures certain Lipschitzness in the ERM training problem. While not exactly requiring Lipschitz continuity in the same way, Lipschitz constants have been used before for measuring training complexity in \cite{goel2017reliably} and more recently have been shown to be linked to generalization by \cite{gouk2018regularisation}.

We point out three important features of this polyhedral encoding. First, it has provable optimality guarantees regarding the ERM problem and a size with
\emph{linear} dependency on the sample size without assuming convexity of the optimization problem. 
Second, the polytope encodes reasonable approximations of all possible data sets that can
be given as an input to the ERM problem. This in particular shows that
our construction is \emph{not} simply discretizing space: if
one considers a discretization of data contained in $[-1,1]^n\times [-1,1]^m$, the total
number of possible data sets of size $D$
is exponential in $D$, which makes the linear dependence on $D$ of the
size of our polytope a remarkable feature. Finally, our approach can be
directly extended to handle commonly used regularizers (\ref{sec:regERM}). For ease of presentation though we omit regularizers throughout our main
discussions.

\begin{remark}
We remark that our goal is to provide new structural results regarding training problems. Converting our approach into a training \emph{algorithm}, while subject of future research, will certainly take considerable efforts. Nonetheless, we will rely on known training algorithms with provable guarantees and their running times for providing a notion of how good our results are. Note that this is a slightly unfair comparison to us, as training algorithms are not data-independent as our encoding.
\end{remark}

\subsubsection{Complexity results for various network architectures.}
We apply our methodology to various well-known neural network architectures by computing and plugging-in the corresponding constituting elements into our unified results. 
We provide an overview of our results in Table~\ref{tab:results}, where \(k\) is the number of layers, $w$ is width of the network, $n/m$ are the input/output dimensions and $N$ is the total number of parameters. 
In all results the node computations are linear with bias term and normalized coefficients, and activation functions with Lipschitz constant at most 1 and with 0 as a fixed point; these include \emph{ReLU, Leaky ReLU, eLU, Tanh}, among others.

\begin{table}
  \centering
  \caption{Summary of results for various architectures. DNN refers to a fully-connected Deep Neural Network, CNN to a Convolutional Neural Network and ResNet to a Residual Network. $G$ is the graph defining the Network and $\Delta$ is the maximum in-degree in $G$.}
  \label{tab:results}
%\vskip 0.15in
\begin{adjustbox}{max width=\textwidth}
  \begin{tabular}[h]{llll}
\hline
    Type  & Loss & Size of polytope & Notes  \\
\hline \hline
DNN & Absolute/Quadratic/Hinge & $O\big(\big( m w^{O(k^2)} /\epsilon \big)^{n+m+N} D \big)$ & $N=|E({G})|$ \\
DNN & Cross Entropy w/ Soft-Max & $O\big(\big( m \log (m) w^{O(k^2)} /\epsilon\big)^{n+m+N} D\big)$ & $N=|E({G})|$ \\
CNN & Absolute/Quadratic/Hinge & $O\big(\big( m w^{O(k^2)} /\epsilon\big)^{n+m+N} D\big)$ & $N \ll |E({G})|$ \\
ResNet & Absolute/Quadratic/Hinge & $O\big(\big( m \Delta^{O(k^2)} /\epsilon\big)^{n+m+N} D\big)$  \\
ResNet & Cross Entropy w/ Soft-Max & $O\big(\big( m \log(m) \Delta^{O(k^2)} /\epsilon\big)^{n+m+N} D\big)$  \\
\hline
  \end{tabular}
 \end{adjustbox}
%\vskip -0.1in
\end{table}

 Certain improvements in the results in Table \ref{tab:results} can be obtained by further specifying if the ERM problem corresponds to \emph{regression} or \emph{classification}. 
 Nonetheless, these improvements are not especially significant and in the interest of clarity and brevity we prefer to provide a unified discussion.

The reader might wonder if the exponential dependence on the other parameters of our polytope sizes can be
improved, namely the input dimension $n+m$, parameter space dimension $N$ and depth $k$. The dependence on the input dimension is unlikely to be improved due to NP-hardness of training problems (\cite{3nodeblumrivest,boob2018complexity}) and obtaining a polynomial dependence on the parameter space dimension or on the depth remains open \cite{arora1611understanding}.\\

The rest of this paper is organized as follows: in Section \ref{sec:preliminaries} we introduce the main tools we use throughout the paper. These include the definition of \emph{treewidth} and a generalization of a result by \cite{bienstock2018lp}. In Section \ref{sec:ERMapprox} we show how multiple ERM problems can be encoded using a single polytope whose size depend only linearly in the sample-size. We also analyze this polytope's structure and show that its face structure are related to each possible ERM problem. In Section \ref{sec:compl-spec-arch} we specialize our results to ERM problems arising from Neural Networks by explicitly computing the resulting polytope size for various common architectures. In Section \ref{sec:ERMNetwork} we show the sparsity of the network itself can be exploited to obtain an improvement in the polyhedral encoding's size. In Section \ref{sec:generalization} we show that our LP \emph{generalizes well}, in the sense that our benign dependency on the sample size allows us to obtain a moderate-sized polyhedron that approximates the \emph{general risk minimization} problem. Finally, in Section \ref{sec:conclusion} we conclude.

\section{Preliminaries}
\label{sec:preliminaries}

In the following let \([n] \doteq \{1, \dots, n\}\) and
\([n]_0 \doteq \{0, \dots, n\}\). Given a graph $H$, we
will use $V(H)$ and $E(H)$ to denote the vertex-set and edge-set of
$H$, respectively, and $\delta_H(u)$ will be the set of edges incident
to vertex $u$. We will need: 
\begin{definition}\label{def:lipschitzconst}
  For $g: \mathcal{K} \subseteq \mathbb{R}^n \rightarrow \mathbb{R}$, we denote its \emph{Lipschitz
  constant with respect to the $p$-norm over $\mathcal{K}$} as
  $\mathcal{L}_p(g)$, satisfying
  \(| g(x) - g(y)| \leq \mathcal{L}_p(g) \| x - y \|_p\) for all $x,y\in \mathcal{K}$ (whenever it exists). 
\end{definition}

We next define the Lipschitz constant of an ERM problem with respect to the infinity norm.

\begin{definition}
Consider the ERM problem \eqref{eq:erm} with parameters $D, \Phi, \ell, f$. We define the \emph{Architecture Lipschitz Constant} $\mathcal{L}(D,\Phi,\ell, f)$ as
\begin{equation}
    \label{eq:Lipschitz-Arch}
    \mathcal{L}(D,\Phi,\ell, f) \doteq \mathcal{L}_\infty(\ell(f(\cdot , \cdot), \cdot ))
\end{equation}
over the domain $\mathcal{K} = [-1,1]^n  \times \Phi \times [-1,1]^m$.
\end{definition}

We emphasize that in \eqref{eq:Lipschitz-Arch} we are considering the data-dependent entries as variables as well, and not only the parameters $\Phi$ as it is usually done in the literature. This subtlety will become clear later.

Finally, in the following let $\expectation(\omega \in \Omega) {\cdot}$ and  $\variance(\omega
\in \Omega){\cdot}$ denote the \emph{expectation} and \emph{variance} with respect to the random variable $\omega \in \Omega$, respectively.

\subsection{Neural Networks}
\label{sec:neural-network}
A neural network can be understood as a function \(f\) defined over a
directed graph that maps inputs \(x \in \R^n\) to \(f(x) \in
\R^m\). The directed graph \(G = (V,E)\), which represents the network
architecture, often naturally decomposes into layers
\(V = \bigcupdot_{i \in [k]_0} V_i\) with
\(V_i \subseteq V\), where \(V_0\) is referred to as the \emph{input layer} and \(V_k\) as the \emph{output layer}. To all
other layers we refer to as \emph{hidden layers}. 

Each vertex \(v \in V_i\) with \(i \in [k]_0\) has an associated set
of \emph{in-nodes} denoted by \(\delta^+(v) \subseteq V\), so that
\((w,v) \in E\) for all \(w \in \delta^+(v)\) and an
associated set of \emph{out-nodes} \(\delta^-(v) \subseteq V\) defined
analogously. If
\(i = 0\), then \(\delta^+(v)\) are the \emph{inputs} (from data) and
if \(i = k\), then \(\delta^-(v)\) are the \emph{outputs} of the
network. These
graphs do neither have to be acyclic (as in the case of
\emph{recurrent neural networks}) nor does the layer decomposition
imply that arcs are only allowed between adjacent layers (as in the
case of \emph{ResNets}). In \emph{feed-forward} networks, however, the graph is assumed to be acyclic.

Each node \(v \in V\) performs a \emph{node computation} \(g_i(\delta^+(v))\), where
\(g_i: \R^{\card{\delta^+(v)}} \rightarrow \R\) with \(i \in [k]\) is typically a smooth
function (often these are linear or affine linear functions) and then
the \emph{node activation} is computed as \(a_i(g_i(\delta^+(v)))\), where
\(a_i : \R \rightarrow \R\) with \(i \in [k]\) is a (not necessarily smooth) function (e.g.,
ReLU activations of the form \(a_i(x) = \max\{0,x\}\)) and the value on
all out-nodes \(w \in \delta^-(v)\) is set to \(a_i(g_i(\delta^+(v)))\)
for nodes in layer \(i \in [k]\). In \emph{feed-forward} networks, we
can further assume that if $v\in V_i$, then $\delta^+(v) \subseteq \cup_{j=0}^{i-1}V_j$, i.e., all arcs move \emph{forward} in the layers.

\subsection{Treewidth}
\label{sec:treedwidth} 
\emph{Treewidth} is an important graph-theoretical concept in
the context of solving optimization problems with \lq{}sparse\rq{}
structure. This parameter is used to measure how
 \emph{tree-like} the graph is, and its use will be the main workhorse behind our results

 \begin{definition} \label{def:treedecomp} 
   A {\em tree-decomposition} (\cite{rs86}) of an undirected graph $G$ is a pair $(T, Q)$
   where $T$ is a tree and $Q = \{ Q_t \, : \, t \in V(T) \}$ is a
   family of subsets of $V(G)$ such that
\begin{enumerate}[(a)]
\item[(i)] For all $v \in V(G)$, the set $\{ t \in V(T) \, : \, v \in Q_t \}$
forms a sub-tree $T_v$ of $T$, and
\item[(ii)] For each $\{u, v \} \in E(G)$ there is a $t \in V(T)$ such
  that $\{u, v \} \subseteq Q_t$, i.e., $t \in T_u \cap T_v$.
\end{enumerate}
The {\em width} of the decomposition is defined as
$\max \left\{ | Q_t | \, : \, t \in V(T) \right\} \, - \, 1$.  The
{\em treewidth of $G$} is the minimum width over all
tree-decompositions of $G$.
\end{definition} 

We refer to the $Q_t$ as \emph{bags} as customary. 
In addition to \emph{width}, another important feature of a tree-decomposition
$(T,Q)$ we use is the \emph{size of the tree-decomposition} given by
$|V(T)|$. 

An alternative definition to Definition \ref{def:treedecomp} of treewidth that the reader might find useful is the following; 
recall that a \emph{chordal} graph is a graph where every induced cycle has length exactly 3.

\begin{definition} \label{treewidthdef2} An undirected graph $G=(V,E)$ has
  \emph{treewidth} $\le \omega$ if there exists a chordal graph
  $H=(V,E')$ with $E\subseteq E'$ and clique number $\le \omega+1$.
\end{definition}

$H$ in the definition above is sometimes referred to as a \emph{chordal completion} of $G$. In Figure \ref{TDexample} we present an example of a graph and a valid tree-decomposition. The reader can easily verify that the conditions of Definition \ref{def:treedecomp} are met in this example. Moreover, using Definition \ref{treewidthdef2} one can verify that the treewidth of the graph in Figure \ref{TDexample} is exactly 2.

\begin{figure}[t!]
    \centering
    \begin{subfigure}{0.5\textwidth}
        \centering
        \begin{tikzpicture}[-,>=stealth',shorten >=1pt,auto,node distance=1.0cm,
          thick,main node/.style={circle,draw,font=\sffamily\small\bfseries, color=orange}]

         \node[main node] (1) at (0,0) {1};
         \node[main node, draw=red] (2) at (1,0) {2};
        \node[main node, color=blue] (3) at (2,0) {3};
        \node[main node, color=green] (4) at (0,-0.8) {4};
        \node[main node, color=teal] (5) at (2,-0.8) {5};
        \node[main node, color=violet] (6) at (0,-1.6) {6};
        \node[main node, color=brown] (7) at (1,-1.6) {7};
        \node[main node, color=cyan] (8) at (2,-1.6) {8};

          \path[every node/.style={font=\sffamily\small}]
            (1) edge (2)
        edge (4)
        (2) edge (3)
        edge (4)
        edge (5)
        edge (7)
        (3) edge (5)
        (4) edge (6)
        edge (7)
        (5) edge (7)
        edge (8)
        (7) edge (6)
        edge (8);
        \end{tikzpicture}
\caption{Graph $G$}
    \end{subfigure}%
    \begin{subfigure}{0.5\textwidth}
        \centering
\begin{tikzpicture}[ - ,>=stealth',shorten >=1pt,auto,node distance=1.5cm,
  thick,main node/.style={circle,draw,font=\sffamily\small\bfseries, color=black,  text width=0.5cm,align=center}]

  \node[main node] (124) at (0,0) {\textcolor{orange}{1}\ \textcolor{red}{2}\\ \textcolor{green}{4}};
  \node[main node] (467) at (3,0) {\textcolor{green}{4}\ \textcolor{violet}{6}\\ \textcolor{brown}{7}};
\node[main node] (247) at (1.5,0) {\textcolor{red}{2}\ \textcolor{green}{4}\\ \textcolor{brown}{7}};
\node[main node] (257) at (1.5,-1.5) {\textcolor{red}{2}\ \textcolor{teal}{5}\\ \textcolor{brown}{7}};
\node[main node] (235) at (0,-1.5) {\textcolor{red}{2}\ \textcolor{blue}{3}\\ \textcolor{teal}{5}};
\node[main node] (578) at (3,-1.5) {\textcolor{teal}{5}\ \textcolor{brown}{7}\\ \textcolor{cyan}{8}};

  \path[every node/.style={font=\sffamily\small}]
(247) edge (124)
edge (467)
edge (257)
(257) edge (235)
edge (578);
\end{tikzpicture}
\caption{A tree-decomposition of $G$ of width 2, with the sets $Q_t$ indicated inside each node of the tree.
}
    \end{subfigure}
    \caption{Example of graph and valid tree-decomposition}
    \label{TDexample}
\end{figure}

Two important folklore results we use are the following.
\begin{lemma} \label{lemma:sizeofTD}
Let $G$ be a graph with a valid tree-decomposition $(T,Q)$ of width $\omega$. Then there exists a valid tree-decomposition $(T',Q')$ of width at most $\omega$ such that $|V(T')| \in O(|V(G)|)$.
\end{lemma}
\begin{lemma} \label{lemma:cliqueinbag}
Let $G$ be a graph with a valid tree-decomposition $(T,Q)$ and $K\subseteq V(G)$ a clique of $G$. Then there exists $t\in T$ such that $K\subseteq Q_t$.
\end{lemma}

\subsection{Binary optimization problems with small treewidth}
Here we discuss how to formulate and solve binary optimization problems that exhibit sparsity in the form of small treewidth. Consider a problem of the form
\begin{subequations}
\begin{align}
(\text{\textbf{BO}}) \qquad \min &\quad c^T x + d^T y\\
\text{s.t.} & \quad f_i(x) \ \ge \ 0 && i\in [m]\\
& \quad g_j(x)\ = \ y_j && j\in [p] \\
& \quad x \in \{0,1\}^n,
\end{align}
\end{subequations}
where the $f_i$ and $g_j$ are arbitrary functions that we access via a
function value oracle.
\begin{definition} \label{intersectgraph} The \emph{intersection
graph} $\Gamma[\mathcal{I}]$ for an instance $\mathcal{I}$ of \textbf{BO} 
is the graph which has a
vertex for each $x$ variable and an edge for each pair of $x$
variables that appear in a common constraint.
\end{definition}

Note that in the above definition we have ignored the $y$ variables
which will be of great importance later. The \emph{sparsity of a
problem} is now given by the treewidth of its intersection graph and we
obtain:

\begin{restatable}{theorem}{genbtheorem} \label{genbtheorem}
Let $\mathcal{I}$ be an instance of \textbf{BO}. If $\Gamma[\mathcal{I}]$ has a tree-decomposition $(T,Q)$ of width $\omega$, there is an exact linear programming reformulation of $\mathcal{I}$ with
  \(O\left( 2^{\omega} \, (|V(T)| + p) \right)\)
variables and constraints.
\end{restatable}

Theorem \ref{genbtheorem} is a generalization of a theorem by \cite{bienstock2018lp} distinguishing
the variables $y$, which do not need to be binary in nature, but are
fully determined by the binary variables $x$. A full proof is
omitted as it is similar to the proof in
\cite{bienstock2018lp}. For the sake of
completeness, we include a proof sketch below.

\begin{proof}(sketch).
Since the support of each $f_i$ induces a clique in the intersection graph, there must exist a bag $Q$ such that $\text{supp}(f_i) \subseteq Q$ (Lemma \ref{lemma:cliqueinbag}). The same holds for each $g_j$. We modify the tree-decomposition $(T,Q)$ to include the $y_j$ variables the following way:
\begin{itemize}
    \item For each $j \in [p]$, choose a bag $Q$ containing $\text{supp}(g_j)$  and add a new bag $Q'(j)$ consisting of $Q\cup \{y_j\}$ and connected to $Q$.
    \item We do this for every $j\in [p]$, with a different $Q'(j)$ for each different $j$. This creates a new tree-decomposition $(T', Q')$ of width at most $\omega + 1$, which has each variable $y_j$ contained in a \emph{single} bag $Q'(j)$ which is a \emph{leaf}.
    \item The size of the tree-decomposition is $|T'| = |T| + p$.
\end{itemize}
From here, we proceed as follows:
\begin{itemize}
    \item For each $t\in T'$, if $Q'_t \ni y_j$ for some $j\in [p]$, we construct
    \begin{equation*}
        \begin{split}
            \mathcal{F}_t \doteq &\, \{(x,y) \in \{0,1\}^{Q_t} \times \mathbb{R} \, : \\
  & y = g_j(x), f_i (x) \geq 0 \text{ for } \text{supp}(f_i) \subseteq Q_t'\}
        \end{split}
    \end{equation*}
otherwise we simply construct
$$\mathcal{F}_t \doteq \{x \in \{0,1\}^{Q_t} \, :
\, f_i (x) \geq 0  \text{ for } \text{supp}(f_i) \subseteq Q_t' \}.$$
Note that these sets have size at most $2^{|Q'_t|}$.
\item We define variables $X[Y,N]$ where $Y,N$ form a partition of $Q'_{t_1} \cap Q'_{t_2}$. These are at most $2^{\omega} |V(T')|$.
\item For each $t\in T'$ and $v\in \mathcal{F}_t$, we create a variable $\lambda_v$. These are at most $2^{\omega} |V(T')|$.
\end{itemize}
We formulate the following \emph{linear} optimization problem
\begin{subequations} \label{LBO}
\begin{align}
(\text{\textbf{LBO}}) \quad \min & \quad c^T x + d^T y\\
\text{s.t.} & \quad \sum_{v\in \mathcal{F}_t} \lambda_v = 1 && \forall t\in T' \\
& \quad X[Y,N] = \sum_{v\in \mathcal{F}_t} \lambda_v \prod_{i\in Y} v_i \prod_{i\in N} (1-v_i)  && \forall (Y,N) \subseteq Q'_t,\, t \in T' \\
& \quad \lambda_v \geq 0  && \forall t\in T',\, v\in \mathcal{F}_t \\
& \quad x_i = \sum_{v\in \mathcal{F}_t} \lambda_v v_i && \forall t\in T',\, i\in Q_t' \cap [n] \\
 & \quad y_j = \sum_{v\in \mathcal{F}_{Q'(j)}} \lambda_v g_j(v) && \forall j\in [p] \\
\end{align}
\end{subequations}
Note that the notation in the last constraint is justified since by construction $\text{supp}(g_j) \subseteq Q'(j)$. The proof of the fact that \textbf{LBO} is equivalent to \textbf{BO} follows from the arguments by \cite{bienstock2018lp}. The key difference justifying the addition of the $y$ variables relies in the fact that they only appear in leaves of the tree decomposition $(T', Q')$, and thus in no intersection of two bags. The gluing argument using variables $X[Y,N]$ then follows directly, as it is then only needed for the $x$ variables to be binary.\\

We can substitute out the $x$ and $y$ variables and obtain a polytope whose variables are only $\lambda_v$ and $X[Y,N]$. This produces a polytope with at most $2\cdot 2^\omega |V(T')|$ variables and $(2\cdot 2^\omega + 1)|V(T')| $ constraints. This proves the size of the polytope is $O(2^{\omega}(|V(T)| + p))$ as required.
\end{proof}

\section{Approximation to ERM via a data-independent polytope}\label{sec:ERMapprox}

We now proceed to the construction of the data-independent polytope encoding multiple ERM problem.
As mentioned before, we assume $\Phi \subseteq [-1,1]^N$ as well as $(x,y) \sim \mathcal D$ satisfies $(x,y) \in [-1,1]^n \times [-1,1]^m$ as normalization to simplify the exposition. Since the \textbf{BO} problem only considers linear objective functions, we begin by reformulating the ERM problem~\eqref{eq:erm} in the following form:
\begin{align}
\label{ERMepigraph}
    \min_{\phi \in \Phi} \left\{\frac{1}{D} \sum_{d = 1}^{D} L_d   \, \middle| \, L_d \, = \, \ell(f(\hat{x}^d, \phi), \hat{y}^d) \quad \forall\, d\in [D] \right\}
  \end{align}

\subsection{Approximation of the feasible region via an $\epsilon$-grid}

Motivated by this reformulation, we study an approximation to the following set:
\begin{flalign}\label{ERMregion}
 S(D, \Phi, \ell, f) = \{(x^1,...\, , x^{D}, y^1, ...\, , y^{D}, \phi, L) :  & L_d = \ell(f(x^d, \phi), y^d),  \\
 & (x^i,y^i) \in [-1,1]^{n+m}, \\
 & \phi \in \Phi\}
\end{flalign}

The variables $(x^i,y^i)_{i=1}^D$ denote the data variables.
Let $r \in \R$ with $-1 \le r \le 1$. Given  $ \gamma \in (0,1)$  we can approximate $r$ as a sum of inverse
powers of $2$, within additive error proportional to $\gamma$. For $N_{\gamma} \ \doteq \ \lceil \log_2 \gamma^{-1}\rceil$
there exist values $z_h \in \{0,1\}$ with $h \in [N_\gamma]$, so that
\begin{equation} \label{sumof2}
   -1 + 2\cdot \sum_{h = 1}^{N_\gamma} 2^{-h}z_h  \, \le \, r \, \le \ -1 + 2\cdot \sum_{h = 1}^{N_\gamma} 2^{-h}z_h \, + \, 2 \gamma \ \le \ 1.
\end{equation}
Our strategy is now to approximately represent the $x,y,\phi$ variables as
$-1 + 2\cdot \sum_{h = 1}^{L_\gamma} 2^{-h}z_{h}$ where each $z_{h}$ is a (new) binary variable. Define $ \epsilon = 2 \gamma \mathcal{L}$, where $\mathcal{L} = \mathcal{L}(D,\Phi, \ell, f)$ is the architecture Lipschitz constant defined in \eqref{eq:Lipschitz-Arch}, and consider the following approximation of $S(D, \Phi, \ell, f)$:
\begin{align*}
    \label{ERMregion-approx}
    S^{\epsilon}(D, \Phi, \ell, f) \doteq  \, \big\{ & \, (x^1, \ldots, x^{D}, y^1, \ldots, y^{D}, \phi,L) \, : \, z \in \{0,1\}^{N_\gamma (N+Dn+Dm)}, \, \phi \in \Phi,\,  \\
   & L_d = \ell(f(x^d, \phi), y^d),\, d\in [D], \\
    & \phi_i = -1 + 2\sum_{h = 1}^{N_\gamma} 2^{-h}z^{\phi}_{i,h},\, i\in [N], \\
    & y^d_i = -1 + 2 \sum_{h = 1}^{N_\gamma} 2^{-h}z^{y^d}_{i,h},\, d\in [D], \, i\in [m], \\
    & x^d_i = -1 + 2 \sum_{h = 1}^{N_\gamma} 2^{-h}z^{x^d}_{i,h},\, d\in [D], \, i\in [n] \big\}.
\end{align*}
Note that substituting out the $x,y, \phi$ using the equations of  $S^\epsilon(D, \Phi, \ell, f)$, we obtain a feasible region as \textbf{BO}. We can readily describe the error of the approximation of $S(D, \Phi, \ell, f)$ by $S^{\epsilon}(D, \Phi, \ell, f)$ in the ERM problem \eqref{eq:erm} induced by the discretization:

\begin{lemma}\label{tobinary}
Consider any \((x^1, \ldots , x^{D}, y^1, \ldots , y^{D}, \phi,L) \in S(D, \Phi, \ell, f) \). Then, there exists \((\hat{x}^1, \ldots , \hat{x}^{D}, \hat{y}^1, \ldots , \hat{y}^{D}, \hat{\phi}, \hat{L})\) \(\in S^\epsilon(D, \Phi, \ell, f)\) such that
\(\left| \frac{1}{D} \sum_{d=1}^D L_d - \frac{1}{D} \sum_{d=1}^D \hat{L}_d \right| \leq \epsilon \).
\end{lemma}

\begin{proof}
Choose binary values $\tilde z$ so as to attain the approximation for variables $x,y,\phi$ as in \eqref{sumof2} and define $\hat x, \hat{y}, \hat{\phi}, \hat{L}$ from $\tilde z$ according to the definition of $S^\epsilon(D, \Phi, \ell, f)$. Since
\[\left\|(x^d,y^d,\phi) -  (\hat x^d, \hat{y}^d, \hat{\phi}) \right\|_\infty \leq 2\gamma = \frac{\epsilon}{\mathcal{L}}  \quad d\in [D]\]
by Lipschitzness we obtain $ |L_d - \hat{L}_d | \leq \epsilon$. The result then follows.
\end{proof}

\subsection{Linear reformulation of the binary approximation} \label{po:reform}

So far, we have phrased the ERM problem~\eqref{eq:erm} as a \textbf{BO} problem using a discretization of the continuous variables. This in and of itself is neither insightful nor useful. In this section we will perform the key step, reformulating the convex hull of $S^\epsilon(D, \Phi, \ell, f)$ as a moderate-sized polytope.

After replacing the $(x,y,\phi)$ variables in $S^\epsilon(D, \Phi, \ell, f)$ using the $z$ variables, we can see that the intersection graph of $S^\epsilon(D, \Phi, \ell, f)$ is given by Figure \ref{fig:ERMintersection}, where we use $(x,y,\phi)$ as stand-ins for corresponding the binary variables $z^x, z^y, z^\phi$. Recall that the intersection graph does not include the $L$ variables. It is not hard to see that a valid tree-decomposition for this graph is given by Figure \ref{fig:ERMtreedecomp}. This tree-decomposition has size $D$ and width $N_\gamma(n+m+N)-1$ (much less than the $N_\gamma (N+Dn+Dm)$ variables). This yields our main theorem:

\begin{figure}[t!]
    \centering
    \begin{subfigure}{0.5\textwidth}
        \centering
        \begin{tikzpicture}[-,>=stealth',shorten >=1pt,auto,node distance=1.2cm,
  thick,main node/.style={circle,draw,font=\sffamily\small\bfseries, color=orange, align=center}, scale=0.9, every node/.style={scale=0.9}]

  \node[main node, color=red, fill=white] (1) at (0,0) {$\phi_1$};
  \node[main node, color=red, fill=white] (2) at (0.7,-0.7) {$\phi_2$};
\node[main node, color=red, fill=white] (3) at (0.5,-1.8) {$\phi_3$};
\node[main node, color=red, fill=white] (4) at (-0.5,-1.8) {$\phi_4$};
\node[main node, color=red, fill=white] (5) at (-0.8,-0.7) {$\phi_N$};

\node[main node, color=blue] (l1) at (-1.6,0.1) {$x^1, {y^1}$};
\node[main node, color=blue] (l2) at (1.6,0.1) {${x^2}, {y^2}$};
\node[main node, color=blue] (l3) at (1.8,-1.7) {${x^3}, {y^3}$};
\node[main node, color=blue] (ld) at (-1.9,-1.7) {${x^D}, {y^D}$};
%\node[main node, color=blue] (l4) at (0,-3.2) {$L_4$\\${x^4}, {y^4}$};

\begin{scope}[on background layer]
\path[every node/.style={font=\sffamily\small}]
    (1) edge (2)
edge (3)
edge(4)
edge(5)
(2) edge (3)
edge (4)
edge (5)
(3) edge (4)
edge (5)
(4) edge(5)
(l1) edge (1)
edge(2)
edge(3)
edge (4)
edge(5)
(l2) edge (1)
edge(2)
edge(3)
edge (4)
edge(5)
(l3) edge (1)
edge(2)
edge(3)
edge (4)
edge(5)
%(l4) edge (1)
%edge(2)
%edge(3)
%edge (4)
%edge(5)
(ld) edge (1)
edge(2)
edge(3)
edge (4)
edge(5);
\end{scope}
\end{tikzpicture}
\caption{Intersection Graph of $S^\epsilon(D, \Phi, \ell, f)$}   \label{fig:ERMintersection}
    \end{subfigure}%
    \begin{subfigure}{0.5\textwidth}
        \centering
\begin{tikzpicture}[ - ,>=stealth',shorten >=1pt,auto,node distance=1.5cm,
  thick,main node/.style={circle,draw,font=\sffamily\small\bfseries, color=blue,  text width=1.0cm,align=center}, scale=0.9, every node/.style={scale=0.9}]

  \node[main node] (1) at (0,0) {$\phi$ \\ $x^{1}, y^{1}$};
\node[main node] (2) at (2,0) {$\phi$ \\ $x^{2}, y^{2}$};
\node (dots) at (3.5,0) {$\cdots$};
\node[main node] (D) at (5,0) {$\phi$ \\ $x^{D}, y^{D}$};

  \path[every node/.style={font=\sffamily\small}]
(1) edge (2)
(2) edge (dots)
(dots) edge (D);
\end{tikzpicture}
\caption{Valid Tree-Decomposition}
    \label{fig:ERMtreedecomp}
    \end{subfigure}
    \caption{Intersection Graph and Tree-Decomposition of $S^\epsilon(D, \Phi, \ell, f)$}
    \label{fig:ERMpictures}
\end{figure}

\begin{maintheorem}\label{maintheorem:ERM}
  Let $D \in \NN$ be a given sample size. Then $\conv(S^\epsilon(D, \Phi, \ell, f))$ is the projection of a polytope with the following properties:
  \begin{enumerate}[(a)]
    \item \label{item:size} The polytope has no more than
\(4 D \left(2\mathcal{L}/\epsilon\right)^{n+m+N} \) variables and \( 2D ( 2 \left(2\mathcal{L}/\epsilon\right)^{n+m+N} + 1 )\) constraints. We refer to the resulting polytope as $P_{S_\epsilon}$.
\item \label{item:time} The polytope $P_{S_\epsilon}$ can be constructed in time $O(\left(2\mathcal{L}/\epsilon\right)^{n+m+N} D)$ plus the time required for $O(\left(2\mathcal{L}/\epsilon\right)^{n+m+N})$ evaluations of $\ell$ and $f$. 
    
\item \label{item:faceExists} For \emph{any} sample $(\hat{X}, \hat{Y}) = (\hat{x}^i, \hat{y}^i)_{i=1}^D$, $(\hat{x}^i,\hat{y}^i) \in [-1,1]^{n+m}$, there is a face $\mathcal{F}_{\hat{X},\hat{Y}}$ of $P_{S_\epsilon}$ such that
\begin{equation*}
    \tilde{\phi} \in  \argmin \, \Big\{\frac{1}{D} \sum_{i=1}^D L_i \quad \Big|  
    \quad (\phi, L) \in \proj_{\phi, L}(\mathcal{F}_{\hat{X},\hat{Y}})\Big\}
\end{equation*}
satisfies
\(\left | \frac{1}{D} \sum_{i=1}^D
  \left(\ell(f(\hat{x}^i,\phi^*), \hat{y}^i) - \ell(f(\hat{x}^i,\tilde \phi), \hat{y}^i) \right)\right|
\leq 2\epsilon,\)
where $\phi^*\in [-1,1]^N$ is an optimal solution to the ERM problem~\eqref{eq:erm} with input data $(\hat{X}, \hat{Y})$. This means that solving an LP using an appropriate face of $P_{S_\epsilon}$ solves the ERM problem~\eqref{eq:erm} within an additive error $2\epsilon$.
\item \label{item:faceConstruct}  The face $\mathcal{F}_{\hat{X},\hat{Y}}$ arises by simply substituting-in actual data for the data-variables $x,y$, which determine the approximations $z^x,z^y$ and is used to fixed variables in the description of $P_{S_\epsilon}$.
  \end{enumerate}
\end{maintheorem}
\begin{proof} Part \eqref{item:size} follows directly from Theorem \ref{genbtheorem} using $N_\gamma = \lceil \log (2\mathcal{L}/\epsilon ) \rceil$ along with the tree-decomposition of Figure \ref{fig:ERMtreedecomp}, which implies $|V(T')| + p = 2D$ in this case. A proof of parts \eqref{item:faceExists} and \eqref{item:faceConstruct} is given in the next subsection. For part \eqref{item:time} we analyze the construction steps of the linear program defined in proof of Theorem \ref{genbtheorem}.

From the tree-decomposition detailed in Section \ref{po:reform}, we see that  data-dependent variables $x,y,L$ are partitioned in different bags for each data $d\in [D]$. Let us index the bags using $d$. Since all data variables have the same domain, the sets $\mathcal{F}_d$ we construct in the proof of Theorem \ref{genbtheorem} will be the same for all $d\in [D]$. Using this observation, we can construct the polytope as follows:
\begin{enumerate}
    \item Fix, say, $d=1$ and enumerate all binary vectors corresponding to the discretization of $x^{1}, y^{1}, \phi$.
    \item Compute $\ell(f(x^{1}, \phi), y^{1})$. This will take $O((2\mathcal{L}/\epsilon)^{n+m+N})$ function evaluations of $f$ and $\ell$. This defines the set $\mathcal{F}_1$.
    \item Duplicate this set $D$ times, and associate each copy with a bag indexed by $d\in [D]$.
    \item For each $d\in [D]$, and each $v\in \mathcal{F}_d$ create a variable $\lambda_v$.
    \item For each $d\in [D-1]$, create variables $X[Y,N]$ corresponding to the intersection of bags $d$ and $d+1$. This will create $O((2\mathcal{L}/\epsilon)^{N})$ variables, since the only variables in the intersections are the discretized $\phi$ variables.
    \item Formulate \textbf{LBO}.
\end{enumerate}
The only evaluations of $\ell$ and $f$ are performed in the construction of $\mathcal{F}_1$. As for the additional computations, the bottleneck lies in creating all $\lambda$ variables, which takes time $O((2\mathcal{L}/\epsilon)^{n+m+N} D)$.

\end{proof}

\begin{remark} \label{remark:input-dimension-improvement}
Note that in step 1 of the polytope construction we are enumerating all possible discretized values
of $x^1,y^1$, i.e., we are implicitly assuming all points in $[-1,1]^{n+m}$ are possible inputs. This is reflected in the $(2\mathcal{L}/\epsilon)^{n+m}$ term in the polytope size estimation. If one were to use another discretization method (or a different ``point generation'' technique) using more information about the input data, this term could be improved and the explicit exponential dependency on the input dimension of the polytope size could be alleviated significantly. However, note that in a fully-connected neural network we have $N\geq n+m$ and thus an implicit exponential dependency on the input dimension could remain unless more structure is assumed. This is in line with the NP-hardness results. We leave the full development of this potential improvement for future work.
\end{remark}

Note that the number of evaluations of $\ell$ and $f$ is independent of $D$.
 We would like to further point out that we can provide an interesting refinement of this theorem: if $\Phi$ has an inherent network structure (as in the Neural Networks case) one can exploit treewidth-based sparsity of the \emph{network itself}.
 This would reduce the exponent in the polytope size to an expression that depends on the \emph{sparsity of the network}, instead of its size. We discuss this in Section \ref{sec:ERMNetwork}. 
\begin{remark}
An additional important point arising from this new perspective on training problems via linear programming comes from duality theory. If one projects-out the variables associated to the parameters $\phi$ in $P_{S_\epsilon}$, the resulting projected polytope would represent all possible samples of size $D$ and their achievable loss vector. This means that there exists a \emph{dual certificate} proving whether a loss vector, or average, is (approximately) achievable by a sample, without using the $\phi$ variables.
\end{remark}

\subsection{Data-dependent faces of the data-independent polytope}
We now proceed to show how the ERM problem for a specific data set is encoded in a face of $P_{S_\epsilon}$. This provides a proof of points \eqref{item:faceExists} and \eqref{item:faceConstruct} in Theorem \ref{maintheorem:ERM}.

Consider a fixed data set $(\hat{X}, \hat{Y}) = (\hat{x}^i, \hat{y}^i)_{i=1}^D$ and let $\phi^*$ be an optimal solution to the ERM problem with input data $(\hat{X}, \hat{Y})$. Since $P_{S_\epsilon}$ encodes ``approximated'' versions of the possible samples, we begin by approximating $(\hat{X}, \hat{Y})$.
 Consider binary variables $z^{\hat{x}}, z^{\hat{y}}$ to attain the approximation \eqref{sumof2} of the input data and define $\tilde{x}, \tilde{y}$ from $z^{\hat{x}}, z^{\hat{y}}$, i.e., $\tilde{x}^d_i = -1 + 2 \sum_{h = 1}^{N_\gamma} 2^{-h}z^{\hat{x}^d}_{i,h}$ and similarly for $\tilde{y}$. Define
\[ S(\tilde{X}, \tilde{Y}, \Phi, \ell, f) = \{ (\phi, L) \in \Phi \times \mathbb{R}^D \, : \, L_d = \ell (f(\tilde{x}^d,\phi) , \tilde{y}^d)\} \]
and similarly as before define $S^\epsilon (\tilde{X}, \tilde{Y}, \Phi, \ell, f)$ to be its discretized version (on variables $\phi$). The following Lemma shows the quality of approximation to the ERM problem obtained using $S(\tilde{X}, \tilde{Y}, \Phi, \ell, f)$ and subsequently $S^\epsilon(\tilde{X}, \tilde{Y}, \Phi, \ell, f)$.
\begin{lemma}
For any $(\phi, L) \in   S(\tilde{X}, \tilde{Y}, \Phi, \ell, f)$ there exists $(\phi', L')\in S^\epsilon (\tilde{X}, \tilde{Y}, \Phi, \ell, f)$ such that
\[ \left| \frac{1}{D} \sum_{d=1}^D L_d - \frac{1}{D} \sum_{d=1}^D L'_d \right| \leq \epsilon.  \]
Additionally, for every $\phi\in \Phi$, there exists  $(\phi', L')\in S^\epsilon (\tilde{X}, \tilde{Y}, \Phi, \ell, f)$ such that
\[ \left| \frac{1}{D} \sum_{d=1}^D \ell(f(\hat{x}^d, \phi), \hat{y}^d) - \frac{1}{D} \sum_{d=1}^D L'_d \right| \leq \epsilon.  \]
\end{lemma}
\begin{proof}
The first inequality follows from the same proof as in Lemma \ref{tobinary}. For the second inequality, let $\phi'$ be the binary approximation to $\phi$, and $L'$ defined by $L'_d = \ell(f(\tilde{x}^d, \phi'), \tilde{y}^d)$. Since $\tilde{x}, \tilde{y}, \phi'$ are approximations to $\hat{x}, \hat{y}, \phi$, the result follows from Lipschitzness.
\end{proof}
\begin{lemma}
\begin{equation*}
    (\hat{\phi}, \hat{L}) \in \argmin \left\{\frac{1}{D} \sum_{d=1}^D L_d : 
     (\phi, L) \in  S^\epsilon (\tilde{X}, \tilde{Y}, \Phi, \ell, f) \right\}
\end{equation*}
satisfies
\[ \left| \frac{1}{D} \sum_{d=1}^D \ell(f(\hat{x}^d, \phi^*), \hat{y}^d) - \frac{1}{D} \sum_{d=1}^D \ell(f(\hat{x}^d, \hat{\phi}), \hat{y}^d) \right| \leq 2\epsilon. \]
\end{lemma}
\begin{proof}
Since $\hat{\phi}\in \Phi$, and $\phi^*$ is an optimal solution to the ERM problem, we immediately have
\( \frac{1}{D} \sum_{d=1}^D \ell(f(\hat{x}^d, \phi^*), \hat{y}^d) \leq \frac{1}{D} \sum_{d=1}^D \ell(f(\hat{x}^d, \hat{\phi}), \hat{y}^d).\)
On the other hand, by the previous Lemma we know there exists  $(\phi', L')\in S^\epsilon (\tilde{X}, \tilde{Y}, \Phi, \ell, f)$ such that
\begin{align}
    -\epsilon \leq\, & \frac{1}{D} \sum_{d=1}^D \ell(f(\hat{x}^d, \phi^*), \hat{y}^d) - \frac{1}{D} \sum_{d=1}^D L'_d  \leq \,   \frac{1}{D} \sum_{d=1}^D \ell(f(\hat{x}^d, \phi^*), \hat{y}^d) - \frac{1}{D} \sum_{d=1}^D \hat{L}_d \label{eq:optimality} \\
     = \, &  \frac{1}{D} \sum_{d=1}^D \ell(f(\hat{x}^d, \phi^*), \hat{y}^d) - \frac{1}{D} \sum_{d=1}^D \ell(f(\tilde{x}^d, \hat{\phi}), \tilde{y}^d) \\
     \leq \, & \frac{1}{D} \sum_{d=1}^D \ell(f(\hat{x}^d, \phi^*), \hat{y}^d) - \frac{1}{D} \sum_{d=1}^D \ell(f(\hat{x}^d, \hat{\phi}), \hat{y}^d)  + \epsilon. \label{eq:lipschitzness}
\end{align}
The rightmost inequality in \eqref{eq:optimality} follows from the optimality of $\hat{L}$ and \eqref{eq:lipschitzness} follows from Lipschitzness.
\end{proof}

Note that since the objective is linear, the optimization problem in the previous Lemma is equivalent if we replace $S^\epsilon (\tilde{X}, \tilde{Y}, \Phi, \ell, f)$ by its convex hull. Therefore the only missing link is the following result. 

\begin{lemma} \label{lemma:faceofpolytope}
$\conv(S^\epsilon (\tilde{X}, \tilde{Y}, \Phi, \ell, f))$ is the projection of a face of $P_{S_\epsilon}$ 
\end{lemma}
\begin{proof}
The proof follows from simply fixing variables in the corresponding \textbf{LBO} that describes $\conv(S^\epsilon (D, \Phi, \ell, f))$. For every $d\in [D]$ and $ v\in \mathcal{F}_d$, we simply need to make $\lambda_v = 0$ whenever the $(x,y)$ components of $v$ do not correspond to $\tilde{X}, \tilde{Y}$. We know this is well defined, since $\tilde{X}, \tilde{Y}$ are already discretized, thus there must be some $v\in \mathcal{F}_d$ corresponding to them.\\

The structure of the resulting polytope is the same as \textbf{LBO}, so the fact that it is exactly $\conv(S^\epsilon (\tilde{X}, \tilde{Y}, \Phi, \ell, f))$ follows. The fact that it is a face of $\conv(S^\epsilon (D, \Phi, \ell, f))$ follows from the fact that the procedure simply fixed some inequalities to be tight.
\end{proof}

\subsection{Data-dependent polytope?}
Before moving to the next section, we would like to discuss the importance of the \emph{data-independent} feature of our construction.
Constructing a polytope for a \emph{specific data set} is trivial: similarly to what we described in the previous sections, with the input data fixed we can simply enumerate over a discretization of $\Phi \subseteq [-1,1]^N$, and thus compute the (approximately) optimal solution in advance. A \emph{data-dependent} polytope would simply be a single vector, corresponding to the approximately optimal solution computed in the enumeration. The time needed to generate such polytope is $O((2\mathcal{L}/\epsilon)^{N})$ (the number of possible discretized configurations) via at most $O((2\mathcal{L}/\epsilon)^{N} D)$ evaluations of $\ell$ and $f$ (one per each enumerated configuration and data-point).

This result is not particularly insightful, as it is based on a straight-forward enumeration which takes a significant amount of time, considering that it only serves one data set. On the other hand, our result shows that by including the input data as a \emph{variable}, we do not induce an exponential term in the size of the data set $D$ and we can keep the number function evaluations to be roughly the same.

\section{Encoding results for feed-forward neural networks}
\label{sec:compl-spec-arch}
We now proceed to deriving explicit results for specific architectures. This amounts to using Theorem \ref{maintheorem:ERM} with explicit computations of the architecture Lipchitz constant $\mathcal L$.

\subsection{Fully-connected layers with ReLU activations and normalized coefficients}
\label{sec:fully-conn-layers}
We consider a Deep Neural Network $f:\mathbb{R}^{n} \rightarrow \mathbb{R}^{m}$ with $k$ layers given by
$f = T_{k} \circ \sigma \circ \cdots \circ T_2 \circ \sigma \circ  T_1$,
where $\sigma$ is the ReLU activation function $\sigma(x) \doteq \max\{0,x\}$ applied component-wise and each $T_i:\mathbb{R}^{w_{i-1}} \rightarrow \mathbb{R}^{w_i}$ is an affine linear function. Here $w_0 = n$ ($w_{k} = m$) is the input (output) dimension of the network. 
We write $T_i(z) = A_i z + b_i$ and assume $\|A_i\|_\infty \leq 1$, $\|b_i\|_{\infty} \leq 1$ via normalization. Thus, if $v$ is a node in layer $i$, the node computation performed in $v$ is of the form $\hat{a}^Tz + \hat{b}$, where $\hat{a}$ is a row of $A_i$ and $\hat{b}$ is a component of $b_i$. Note that in this case the  parameter space dimension is exactly the number of edges of the network. Hence, we use $N$ to represent the number of edges. We begin with a short technical Lemma, with which we can immediately establish the following corollary. 
\begin{lemma} \label{lemma:domainLayers}
For every $i\in [k-1]_0$ define $U_i = \sum_{j=0}^{i} w^j$. If $\| z \|_\infty \leq U_{i}$ then \( \| T_{i+1}(z) \|_{\infty} \leq U_{i+1} \). 
\end{lemma}

\begin{proof}
The result can be verified directly, since for $a\in [-1,1]^w$ and $b\in [-1,1]$ it holds $|z^T a + b| \leq w \| z \|_{\infty} + 1$.
\end{proof}

\begin{corollary} \label{cor:DNNtraining}
If $\Phi$ is the class of Neural Networks with $k$ layers, $N$ edges, ReLU activations, and normalized coefficients, then $\conv(S^\epsilon(D, \Phi, \ell, f))$ can be formulated via a polytope of size
\(O( ( 2\mathcal{L_\infty}(\ell) w^{O(k^2)} / \epsilon )^{n+m+N} D ),\)
where $w = \max_{i\in [k-1]_0} w_i$ and $\mathcal{L_\infty}(\ell)$ is the Lipschitz constant of $\ell(\cdot, \cdot)$ over $[-U_{k}, U_{k}]^m \times [-1,1]^m$. The polytope can be constructed in time $O(( 2\mathcal{ L_\infty}(\ell) w^{O(k^2)} / \epsilon )^{n+m+N} D )$  plus the time required for $O(( 2\mathcal{L_\infty}(\ell) w^{O(k^2)} / \epsilon )^{n+m+N})$ evaluations of $\ell$ and $f$.
\begin{proof}
Proving that the architecture Lipschitz constant is $L_\infty(\ell) w^{O(k^2)}$ suffices. All node computations take the form $h(z,a,b) = z^T a + b$ for $a \in [-1,1]^{w}$ and $b\in [-1,1]$; the only difference is made in the domain of $z$, which varies from layer to layer. The $1$-norm of the gradient of $h$ is at most $\|z\|_1 +\|a\|_1 + 1 \leq \| z \|_1 + w + 1$ which, in virtue of Lemma \ref{lemma:domainLayers}, implies that a node computation on layer $i$ (with the weights considered variables) has Lipschitz constant at most
$   % 
\sum_{j=0}^i 2w^j = 2 \frac{w^{i+1} - 1}{w - 1}  =: \tilde{w}^{i}.
$
On the other hand, for $\|(a,b) - (a',b')\|_\infty \leq \gamma $ and $z\in [-U_i, U_i]$, it holds that
\begin{align*} | h(z,a,b) - h(z',a',b') | \leq &\, \tilde{w}^{i} \| (z - z', a - a', b - b') \|_\infty\\
\leq &\, \tilde{w}^{i}  \max\{\| z-z' \|_\infty, \gamma\}
\end{align*}
which shows that the Lipschitz constants can be multiplied layer-by-layer to obtain the overall architecture Lipschitz constant. Since ReLUs have Lipschitz constant equal to 1, and \(\prod_{i=1}^{k} \tilde{w}^{i} = \prod_{i=1}^{k} \big( 2 \frac{w^{i+1} - 1}{w - 1}\big) = w^{O(k^2)}, \)
whenever $w\geq 2$, we conclude the architecture Lipschitz constant is $L_\infty(\ell) w^{O(k^2)}$.
\end{proof}
\end{corollary}

To evaluate the quality of the polytope size in the previous lemma, we compare with the following related algorithmic result.
\begin{theorem}{\cite[Theorem 4.1]{arora1611understanding}}\label{theo:arora}
Let $\Phi$ be the class of Neural Networks with 1 hidden layer ($k=2$), convex loss function $\ell$, ReLU activations and output dimension $m=1$. There exists an algorithm to find a global optimum of the ERM problem in time
\( O(2^w D^{nw} \operatorname{poly}(D,n,w)) \).
\end{theorem}
In the same setting, our result provides a polytope of size
\begin{equation}\label{eq:oursize}
    O( ( 2\mathcal{L_\infty}(\ell) w^{O(1)} / \epsilon )^{(n+1)(w+1)} D )
\end{equation}

\begin{remark}
We point out a few key differences of these two results:
\begin{inparaenum}[(a)]
    \item One advantage of our result is the benign dependency on $D$. 
    Solving the training problem using an LP with our polyhedral encoding has polynomial dependency on the data-size regardless of the architecture. Moreover, our approach is able to construct a polytope that would work for any sample.
    \item The exponent in \eqref{eq:oursize} is $\sim nw$, which is also present in Theorem \ref{theo:arora}. The key difference is that we are able to swap the base of that exponential term for an expression that does not depend on $D$.
    \item We are able to handle any output dimension $m$ and any number of layers $k$.
    \item We do not assume convexity of the loss function $\ell$, which causes the resulting polytope size to depend on how well behaved $\ell$ is in terms of its Lipschitzness.
    \item The result of \cite{arora1611understanding} has two advantages over our result: there is no boundedness assumption on the coefficients, and they are able to provide a \emph{globally optimal} solution.
\end{inparaenum}
\end{remark}

\subsection{ResNets, CNNs, and alternative activations}
\label{sec:resnets-with-relu}
Corollary~\ref{cor:DNNtraining} can be generalized to handle other architectures as well,
as the key features we used before are the acyclic structure of the network and the Lipschitz constant of the ReLU function. 

\begin{lemma} \label{lemma:fullyConnected}
Let $\Phi$ be the class of feed-forward neural networks with $k$ layers, 
$N$ edges, affine node computations, 1-Lipschitz activation functions $a_i: \mathbb{R} \rightarrow \mathbb{R}$ such that $a_i(0) = 0$, and normalized coefficients. Then $\conv(S^\epsilon(D, \Phi, \ell, f))$ can be formulated via a polytope of size
\(O((  2\mathcal{L_\infty}(\ell) \Delta^{O(k^2)} / \epsilon)^{n+m+N} D),\)
where $\Delta$ is the maximum vertex in-degree and $\mathcal{L_\infty}(\ell)$ is the Lipschitz constant of $\ell(\cdot, \cdot)$ over $[-U_{k}, U_{k}]^m \times [-1,1]^m$. The polytope can be constructed in time $O(( 2\mathcal{ L_\infty}(\ell) \Delta^{O(k^2)} / \epsilon )^{n+m+N} D )$ plus the time required for $O(( 2\mathcal{ L_\infty}(\ell) \Delta^{O(k^2)} / \epsilon )^{n+m+N} )$ evaluations of $\ell$ and $f$.
\end{lemma}
\begin{proof}
The proof follows almost directly from the proof of Corollary~\ref{cor:DNNtraining}. The two main differences are (1) the input dimension of a node computation, which can be at most $\Delta$ instead of $w$ and (2) the fact that an activation function $a$ with Lipchitz constant 1 and that $a(0) = 0$ satisfies $|a(z)| \leq |z|$, thus the domain of each node computation computed in Lemma \ref{lemma:domainLayers} applies. The layer-by-layer argument can be applied as the network is feed-forward.
\end{proof}

\begin{corollary}
The ERM problem \eqref{eq:erm} over Deep Residual Networks (ResNets) with 1-Lipschitz activations can be solved to $\epsilon$-optimality in time $\operatorname{poly}(\Delta, 1/\epsilon, D)$ whenever the network size and number of layers are fixed.
\end{corollary}

Another interesting point can be made with respect to Convolutional Neural Networks (CNN). In these, convolutional layers are included to significantly reduce the number of parameters involved. From a theoretical perspective, a CNN can be obtained by enforcing certain parameters of a fully-connected DNN to be equal. This implies that Lemma \ref{lemma:fullyConnected} can also be applied to CNNs, with the key difference residing in parameter $N$, which is the dimension of the parameter space and \emph{does not} correspond to the number of edges in a CNN.

\subsection{Explicit Lipschitz constants of common loss functions}
\label{sec:LipschitzConstant}
In the previous section we specified our results ---the size of the data-independent polytope--- for feed-forward networks with 1-Lipschitz activation functions. However, we kept as a parameter $\mathcal{L_\infty}(\ell)$; the Lipschitz constant of $\ell(\cdot, \cdot)$ over $[-U_{k}, U_{k}]^m \times [-1,1]^m$, with $U_{k} = \sum_{j=0}^{k} w^j$ a valid bound on the output of the node computations, as proved in Lemma \ref{lemma:domainLayers}. Note that $U_{k} \leq w^{k+1}$.

In this section we compute this Lipschitz constant for various common loss functions. It is important to mention that we are interested in the Lipschitznes of $\ell$ with respect to both the output layer and the data-dependent variables as well ---not a usual consideration in the literature. These computations lead to the results reported in Table \ref{tab:results}.

Recall that a bound on the Lipschitz constant $\mathcal{L_\infty}(\ell)$ is given by $\sup_{z,y} \| \nabla \ell(z,y) \|_1$.
\begin{itemize}
    \item Quadratic Loss $\ell(z,y) = \| z - y \|_2^2$. In this case  it is easy to see that
    \[ \| \nabla \ell(z,y) \|_1 = 4 \| z - y\|_1 \leq 4 m (U_{k} +1) \leq 4m(w^{k+1} + 1)\]
    \item Absolute Loss $\ell(z,y) = \| z - y \|_1$. In this case we can directly verify that the Lipschitz constant with respect to the infinity norm is at most $2m$.
    \item Cross Entropy Loss with Soft-max Layer. In this case we include the Soft-max computation in the definition of $\ell$, therefore
    \[\ell(z,y) = -\sum_{i=1}^m y_i \log(S(z)_i) \]
    where $S(z)$ is the Soft-max function defined as 
    \[ S(z)_i = \frac{e^{z_i}}{\sum_{j=1}^m e^{z_j}}. \]
    A folklore result is
    \[ \frac{\partial \ell (z,y)}{\partial z_i} = S(z)_i - y_i \Rightarrow \left| \frac{\partial \ell (z,y)}{\partial z_i} \right| \leq 2. \]
    Additionally, 
    \[ \frac{\partial \ell (z,y)}{\partial y_i} =  -\log(S(z)_i)  \]
    which in principle cannot be bounded. Nonetheless, since we are interested in the domain $[-U_{k}, U_{k}]$ of $z$, we obtain
    \[ S(z)_i = \frac{e^{z_i}}{\sum_{j=1}^m e^{z_j}} \geq \frac{1}{m}e^{-2U_{k}} \] \[\Rightarrow \left| \frac{\partial \ell (z,y)}{\partial y_i} \right| = -\log(S(z)_i) \leq \log(m) + 2U_{k} \]
    which implies that $\mathcal{L_\infty}(\ell) \leq 2m(\log(m) + 2U_{k}) \leq 2m(\log(m) + 2w^{k+1}) $.
    \item Hinge Loss $\ell (z,y) = \max\{ 1 - z^T x , 0\}$. Using a similar argument as for the Quadratic Loss, one can easily see that the Lipschitz constant with respect to the infinity norm is at most $m(U_{k} + 1) \leq m (w^{k+1} + 1)$.
\end{itemize}

\section{ERM under Network Structure} \label{sec:ERMNetwork}
So far we have considered general ERM problems exploiting only the structure of the ERM induced by the finite sum formulations. We will now study ERM under Network Structure, i.e., specifically ERM problems as they arise in the context of Neural Network training.
We will see that in the case of Neural Networks, we can exploit the sparsity of the network itself to obtain better polyhedral formulations of $\conv(S^\epsilon(D, \Phi, \ell, f))$. 

Suppose the network is defined by a graph $\mathcal{G}$, and recall that in this case, $\Phi \subseteq [-1,1]^{E(\mathcal{G})}$. By using additional auxiliary variables $s$ representing the node computations and activations, we can describe $S(D, \Phi, \ell, f)$ in the following way:
\begin{align*}
    S(D, \Phi, \ell, f) =& \big\{(x^1, \ldots, x^{D}, y^1, \ldots, y^{D}, \phi,L)\, : \, \\
    & L_d = \ell(s^{k,d}, y^d)  \\
    & s^{i,d}_v = a_v(g_v(s^{i-1,d},\phi(\delta^+(v))) \, \forall v\in V_i, i\in [k]\\
    & s^{0,d} = x^{d} \\
    & x^i \in [-1,1]^n,\, y^i \in [-1,1]^m, \phi \in \Phi \big\}.
\end{align*}
The only difference with our original description of $S(D, \Phi, \ell, f)$ in \eqref{ERMregion} is that we explicitly ``store'' node computations in variables $s$. These new variables will allow us to better use the structure of $\mathcal{G}$.

\begin{assumption} \label{assumption:boundedcomputations}
To apply our approach in this context we need to further assume $\Phi$ to be the class of Neural Networks with normalized coefficients and \emph{bounded node computations}. This means that we restrict to the case when $s\in[-1,1]^{|V(\mathcal{G})| D}$.
\end{assumption}

Under Assumption \ref{assumption:boundedcomputations} we can easily derive an analog description of $S^\epsilon(D, \Phi, \ell, f)$ using this node-based representation of $S^\epsilon(D, \Phi, \ell, f)$. In such description we also include a binary
representation of the auxiliary variables $s$. Let $\Gamma$ be the intersection graph of such a formulation of $S^\epsilon(D, \Phi, \ell, f)$ and $\Gamma_\phi$ be the sub-graph of $\Gamma$ induced by variables $\phi$. Using a tree-decomposition $(T,Q)$ of $\Gamma_\phi$ we can construct a tree-decomposition of $\Gamma$ the following way:
\begin{enumerate}
    \item We duplicate the decomposition $D$ times $(T^i, Q^i)_{i=1}^{D}$, where each $(T^i,Q^i)$ is a copy of $(T,Q)$.
    \item We connect the trees $T^i$ in a way that the resulting graph is a tree (e.g., they can be simply concatenated one after the other).
    \item To each bag $Q^i_t$ with $t\in T^i$ and $i \in [D]$, we add all the data-dependent variables $L_d$ and the binary variables associated with the discretization of $x^d, s^{\cdot, d}$, and $y^d$. This adds $N_\gamma(|V(\mathcal{G})| + n + m)$ additional variables to each bag, as there is only one variable $s$ per data point per vertex of $\mathcal{G}$.
\end{enumerate}

It is not hard to see that this is a valid tree-decomposition of $\Gamma$, of size $|T|\cdot D$ ---since the bags were duplicated $D$ times--- and width $N_\gamma( tw(\Gamma_\phi) + |V(\mathcal{G})| + n + m)$. 

We now turn to providing a bound to $tw(\Gamma_\phi)$. To this end we observe the following:
\begin{enumerate}
    \item The architecture variables $\phi$ are associated to edges of $\mathcal{G}$. Moreover, two variables $\phi_{e},\phi_{f}$, with $e,f\in E$ appear in a common constraint if and only if there is a vertex $v$ such that $e,f \in \delta^+(v)$.
    \item This implies that $\Gamma_\phi$ is a sub-graph of the \emph{line graph} of $\mathcal{G}$. Recall that the line graph of a graph $\mathcal{G}$ is obtained by creating a node for each edge of $\mathcal{G}$ and connecting two nodes whenever the respective edges share a common endpoint.
\end{enumerate}  

The treewidth of a line graph is related to the treewidth of the base graph (see \cite{bienstock1990embedding, calinescu1998multicuts, atserias2008digraph, harvey2018treewidth}). More specifically, $tw(\Gamma_\phi) \in O(tw(\mathcal{G}) \Delta(\mathcal{G}))$  where $\Delta$ denotes the maximum vertex degree. Additionally, using Lemma \ref{lemma:sizeofTD} we may assume $|T| \leq |E(\mathcal{G})|$, since $\Gamma_\phi$ has at most $|E(\mathcal{G})|$ nodes. Putting everything together we obtain:

\begin{lemma} \label{lemma:treewidthNN}
If there is an underlying network structure $\mathcal{G}$ in the ERM problem and the node computations are bounded, then $\conv(S^\epsilon(D, \Phi, \ell, f))$ is the projection of a polytope with no more than
\[2 D (|E(\mathcal{G})| + 1) \left(\frac{2\mathcal{L}}{\epsilon}\right)^{O(tw(\mathcal{G}) \Delta(\mathcal{G}) + |V(\mathcal{G})| + n + m)}\] 
variables and no more than
\[ D (|E(\mathcal{G})| + 1)  \left( 2 \left(\frac{2\mathcal{L}}{\epsilon}\right)^{O(tw(\mathcal{G}) \Delta(\mathcal{G}) + |V(\mathcal{G})| + n + m)} + 1 \right)  \]
constraints. Moreover, given a tree-decomposition of the network $\mathcal{G}$, the polytope can be constructed in time 
\[O\left(D |E(\mathcal{G})| \left(2\mathcal{L}/\epsilon\right)^{O(tw(\mathcal{G}) \Delta(\mathcal{G}) + |V(\mathcal{G})| + n + m)}\right)\]
plus the time required for 
\[O\left(|E(\mathcal{G})| \left(2\mathcal{L}/\epsilon\right)^{O(tw(\mathcal{G}) \Delta(\mathcal{G}) + |V(\mathcal{G})| + n + m)}\right)\]
evaluations of $\ell$ and $f$.
\end{lemma}

\section{Linear Programming-based Training Generalizes}
\label{sec:generalization}

In this section we show that the ERM solutions obtained via LP generalize to the General Risk Minimization
problem. Here we show generalization as customary in stochastic
optimization, exploiting the Lipschitzness of the model to be trained;
we refer the interested reader to \cite{shapiro2009lectures} and
\cite{ln:shabbirStochOpt} for an in-depth
discussion. 

Recall that the \emph{General Risk Minimization (GRM)} is
defined as $\min_{\phi \in \Phi} \GRM(\phi) \doteq \min_{\phi \in \Phi} \expectation((x,y) \in \D) {\ell(f(x,\phi),
    y)}$, where \(\ell\) is some \emph{loss function}, \(f\) is a
neural network architecture  with parameter space \(\Phi\), and
\((x,y) \in \R^{n + m}\) drawn from the
distribution \(\D\). 
We solve the ERM problem
$  \min_{\phi \in \Phi} \ERM_{X,Y}(\phi) \doteq \min_{\phi \in \Phi} \frac{1}{D} \sum_{i=1}^{D}
  \ell(f(x^i,\phi), y^i),
$
instead, where \( (X,Y) = (x^i, y^i)_{i=1}^D \) is an i.i.d.~sample from data distribution \(\D\)
of size \(D\). 
We show in this section, for any \(1 > \alpha > 0\), \(\epsilon > 0\), we can choose a (reasonably
small!) sample size \(D\), so that with probability
\(1-\alpha\) it holds:
\[\GRM(\bar \phi) \leq \min_{\phi \in \Phi} \GRM(\phi) +
  6 \epsilon,\]
where \(\bar \phi \leq \max_{\phi \in \Phi} \ERM_{X,Y}(\phi) +
\epsilon\) is an \(\epsilon\)-approximate solution to \(\ERM_{X,Y}\)
for i.i.d.-sampled data \((X,Y) \sim \D\). As the size of the polytope
that we use for training only \emph{linearly depends on $D$},
this also implies that we will have a linear
program of reasonable size as a function of \(\alpha\) and \(\epsilon\).

The following proposition summaries the generalization argument used
in stochastic programming as presented in \cite{ln:shabbirStochOpt}
(see also \cite{shapiro2009lectures}). Let \(\sigma^2 = \max_{\phi \in \Phi} \variance((x,y) \in \mathcal D){\ell(f(x,\phi),
    y)}\).

\begin{proposition}   \label{prop:finiteGeneralization}
Consider the optimization problem
 \begin{equation}
    \label{eq:expectation}
    \min_{x \in X} \expectation(\omega \in \Omega) {F(x,\gamma(\omega))},
  \end{equation}
where \(\gamma(\omega)\) is a random parameter with \(\omega \in
\Omega\) a set of parameters, \(X \subseteq \R^n\) a finite set, and \(F: X \times \Omega
\rightarrow \mathbb R\) is a function. Given
i.i.d.~samples \(\gamma_1, \dots, \gamma_D\) of \(\gamma(\omega)\),
consider the finite sum problem
\begin{equation}
  \label{eq:finSumVersion}
\min_{x \in X} \frac{1}{D} \sum_{i \in [D]} F(x,\gamma_i).
\end{equation}
If \(\bar x \in X\) is an
\(\epsilon\)-approximate solution to \eqref{eq:finSumVersion}, i.e.,
\(\frac{1}{D} \sum_{i \in [D]} F(\bar x,\gamma_i) \leq \min_{x \in X}
\frac{1}{D} \sum_{i \in [D]} F(x,\gamma_i) + \epsilon\) and
\begin{equation}
  \label{eq:sizeBound}
  D \geq \frac{4 \sigma^2}{\epsilon^2} \log \frac{\card{X}}{\alpha},
\end{equation}
where \(\alpha > 0\) and \(\sigma^2 = \max_{x \in X} \variance(\omega
\in \Omega){F(x,\gamma(\omega))}\),
then with probability \(1-\alpha\) it holds:
\begin{equation}
  \label{eq:generalization}
  \expectation(\omega \in \Omega) {F(\bar x,\gamma(\omega))}  \leq \min_{x \in X} \expectation(\omega \in \Omega) {F(x,\gamma(\omega))}
  + 2 \epsilon.
\end{equation}
\end{proposition}

We now establish generalization by means of Proposition
\ref{prop:finiteGeneralization} and a straightforward discretization
argument. By assumption from above $\Phi \subseteq [-1,1]^N$ for
some $N \in \mathbb N$. Let
$\Phi_\nu \subseteq \Phi \subseteq [-1,1]^N$ be a $\nu$-net of
$\Phi$, i.e., for all $\phi \in \Phi$ there exists
$\bar \phi \in \Phi_\nu$ with $\|\phi - \bar \phi\|_\infty \leq
\nu$. Furthermore let $\cL$ the be architecture Lipschitz constant, as defined in \eqref{eq:Lipschitz-Arch} (or  \eqref{eq:Lipschitz-Arch-regularizers}).

\begin{restatable}{theorem}{generalization}[Generalization]
  \label{thm:generalization}
	Let $\bar \phi \in \Phi$ be an $\epsilon$-approximate solution to $\min_{\phi \in \Phi}\ERM_{X,Y}(\phi)$ with $\epsilon > 0$, i.e., $
		\ERM_{X,Y}(\bar \phi) \leq \min_{\phi \in \Phi} \ERM_{X,Y}(\phi) + \epsilon.$ 
If $D \geq \frac{4 \sigma^2}{\epsilon^2} \log
        \frac{((2\cL)/\epsilon)^N}{\alpha}, $
with \(\cL\) and \(\sigma^2\) as above, then with probability $1-\alpha$ it holds
$\GRM(\bar \phi) \leq \min_{\phi \in \Phi} \GRM(\phi) + 6 \epsilon,$
i.e., $\bar \phi$ is a $6\epsilon$-approximate solution to
$\min_{\phi \in \Phi} \GRM(\phi)$.
\end{restatable}

\begin{proof}
  Let $\bar \phi$ be as above. With the choice
  $\nu \doteq \epsilon / \cL$, there exists $\tilde \phi \in \Phi_\nu$,
  so that $\|\tilde \phi - \bar \phi \|_\infty \leq \nu$ and hence by
  Lipschitzness,
	$$| \ERM_{X,Y}(\bar \phi) - \ERM_{X,Y}(\tilde \phi) | \leq \epsilon,$$
	so that
        $\ERM_{X,Y}(\tilde \phi) \leq \min_{\phi \in \Phi_\nu}
        \ERM_{X,Y}(\phi) + 2 \epsilon$. As
        \(D \geq \frac{4 \sigma^2}{\epsilon^2} \log
        \frac{((2\cL)/\epsilon)^N}{\alpha}\), with probability \(1-\alpha\)
        we have
        \(\GRM(\tilde \phi) \leq \min_{\phi \in \Phi_\nu} \GRM(\phi) +
        4 \epsilon\) by
        Proposition~\ref{prop:finiteGeneralization}. If now
        \(\bar \phi_G = \argmin_{\phi \in \Phi} \GRM(\phi)\) and
        \(\tilde \phi_G \in \Phi_\nu\) with
        \(\|\bar \phi_G - \tilde \phi_G\|_\infty \leq \nu\), by Lipschitzness
        we have
        \(| \GRM(\bar \phi_G) - \GRM(\tilde \phi_G) | \leq
        \epsilon\). Now
        \begin{align*}
\GRM(\tilde \phi) & \leq \min_{\phi \in \Phi_\nu}
        \GRM(\phi) + 4 \epsilon \\
          & \leq \GRM(\tilde \phi_G) + 4 \epsilon & \text{(by
                                                       optimality)} \\
          & \leq \GRM(\bar \phi_G) + 5 \epsilon & \text{(by
                                                     Lipschitzness)}.
        \end{align*}
Together with \(| \GRM(\bar \phi) - \GRM(\tilde \phi) | \leq
        \epsilon\) as $\|\tilde \phi - \bar \phi \|_\infty \leq \nu$ it
        follows
        \[\GRM(\bar \phi) \leq \GRM(\bar \phi_G) + 6 \epsilon =
          \min_{\phi \in \Phi} \GRM(\phi)+ 6 \epsilon,\]
which completes the proof.
\end{proof}

We are ready to formulate the following corollary combining Theorem~\ref{thm:generalization} and Main Theorem~\ref{maintheorem:ERM}.

\begin{corollary}[LP-based Training for General Risk Minimization]
    \label{cor:generalization} Let $\mathcal D$ be a data distribution as above. Further, let $1 > \alpha > 0$ and $\epsilon >0$, then there exists a linear program with the following properties:
  \begin{enumerate}[(a)]
    \item The LP has size 
    \[O\Big(\left(2\mathcal{L}/\epsilon \right)^{n+m+N} \left(\frac{4 \sigma^2}{\epsilon^2} \log
            \frac{(2\cL/\epsilon)^N}{\alpha}\right)\Big)\]
        and can be constructed in time 
        \[O\Big(\left(2\mathcal{L}/\epsilon\right)^{n+m+N} \left(\frac{4 \sigma^2}{\epsilon^2} \log
                    \frac{((2\cL)/\epsilon)^N}{\alpha}\right)\Big)\] 
        plus the time required for $O\big(\left(2\mathcal{L}/\epsilon\right)^{n+m+N}\big)$ evaluations of $\ell$ and $f$, where  \(\cL\) and \(\sigma^2\) as above.
    \item With probability $(1-\alpha)$ it holds $\GRM(\bar \phi) \leq \min_{\phi \in \Phi} \GRM(\phi) + 6 \epsilon$,
    where $\bar \phi$ is an optimal solution to the linear program obtained for the respective sample of $\mathcal D$ of size $\frac{4 \sigma^2}{\epsilon^2} \log
            \frac{((2\cL)/\epsilon)^N}{\alpha}$.
  \end{enumerate}
\end{corollary}

Similar corollaries hold, combining Theorem~\ref{thm:generalization} with the respective alternative statements from Section~\ref{sec:compl-spec-arch}. Of particular interest for what follows is the polytope size in the case of a neural network with $k$ layers with width $w$, which becomes
\begin{equation}
    \label{eq:generalizationNN}
    O\big(  \big(  2\mathcal{L_\infty}(\ell) w^{O(k^2)} / \epsilon \big)^{n+m+N} \big( 4 \sigma^2 / \epsilon^2 \big) \log (
             (2\mathcal{L_\infty}(\ell) w^{O(k^2)} /\epsilon)^N / \alpha  ) \big).
\end{equation}

A closely related result regarding an approximation to the GRM problem for neural networks is provided by \cite{goel2017reliably} in the improper learning setting. The following corollary to \cite{goel2017reliably} (Corollary 4.5) can be directly obtained, rephrased to match our notation:

\begin{theorem}[\cite{goel2017reliably}]
\label{thm:goel}
There exists an algorithm that outputs $\tilde{\phi}$ such that with probability $1-\alpha$, 
for any distribution $\mathcal{D}$ and 
loss function $\ell$ which is convex, $L$-Lipschitz in the first argument and $b$ bounded on
$[-2\sqrt{w}, \sqrt{w}]$, $
		\GRM(\tilde \phi) \leq \min_{\phi \in \Phi} \GRM(\phi) + \epsilon$,
where $\Phi$ is the class of neural networks with $k$ hidden layers, width $w$, output dimension $m=1$, ReLU activations and normalized weights. The algorithm runs in time at most
\begin{equation}\label{eq:goel}
 n^{O(1)} 2^{((L+1) w^{k/2} k \epsilon^{-1})^k} \log(1/\alpha) 
 \end{equation}
\end{theorem}
\begin{remark} \label{remark:comparisonGoel}
In contrast to the result of \cite{goel2017reliably}, we consider the proper learning setting, where we actually obtain a neural network. In addition we point out key differences between Theorem~\ref{thm:goel} and the algorithmic version of our result when solving the LP in Corollary~\ref{cor:generalization} of size as \eqref{eq:generalizationNN}:
\begin{inparaenum}[(a)]
\item In \eqref{eq:goel}, the dependency on $n$ is better than in \eqref{eq:generalizationNN}.
\item The dependency on the Lipschitz constant is significantly better in \eqref{eq:generalizationNN}, although we are relying on the Lipschitz constant with respect to \emph{all} inputs of the loss function and in a potentially larger domain.
\item The dependency on $\epsilon$ is also better in \eqref{eq:generalizationNN}.
\item We are not assuming convexity of $\ell$ and we consider general $m$.
\item The dependency on $k$ in \eqref{eq:generalizationNN} is much more benign than the one in \eqref{eq:goel}, which is doubly exponential.
\end{inparaenum}
\end{remark}

\begin{remark} 
A recent manuscript by \cite{manurangsi2018computational} 
provides a similar algorithm to the one by \cite{goel2017reliably} but in the \emph{proper} learning setting, for depth-2 ReLU networks with convex loss functions. The running time of the algorithm (rephrased to match our notation) is $(n/\alpha)^{O(1)} 2^{(w/\epsilon)^{O(1)}}$. Analogous to the comparison in Remark \ref{remark:comparisonGoel}, we obtain a much better dependence with respect to $\epsilon$ and we do not rely on convexity of the loss function or on constant depth of the neural network.
\end{remark}

\section{Conclusion and final remarks} \label{sec:conclusion}

We have showed that ERM problems admit a representation which encodes \emph{all} possible training problems in a single polytope whose size depends only linearly on the sample size and possesses optimality guarantees. Moreover, we show that training
is closely related to the face structure of this \emph{data-independent} polytope. As a byproduct, our contributions also improve some of the best known algorithmic results for neural network training with optimality/approximation guarantees. 

These results shed new light on (theoretical) neural network training by bringing together concepts of graph theory, polyhedral geometry, and non-convex optimization as a tool for Deep Learning. Our data-independent polyhedral encoding, its data-dependent face structure, and the fact that its size is only \emph{linear} on the sample size reveal an interesting interaction between different training problems.

While a straightforward algorithmic use of our formulation is likely to be difficult to solve in practice, we believe the theoretical foundations
we lay here can also have practical implications in the Machine Learning community. 
All our architecture dependent terms are worst-case bounds, which
can be improved by assuming more structure. Additionally, the history of Linear
Programming has provided many important cases of extremely large LPs that can be solved to near-optimality without necessarily generating the complete description. In these, the theoretical understanding of the polyhedral structure is crucial to drive the development of solution strategies.

\section*{Acknowledgements}
Research reported in this paper was partially supported by NSF CAREER award CMMI-1452463, ONR award GG012500 and the Institute for Data Valorization (IVADO). We would also like to thank Shabbir Ahmed for the helpful pointers and discussions.

\bibliographystyle{plain}

\bibliography{bibliography}

%\newpage
\appendix

\section{Further definitions}
\label{sec:furtherDefs}

\subsection{Proper vs.~improper learning}
\label{sec:properDef}

An important distinction is the \emph{type} of solution to the ERM that we allow. In this work we are considering \emph{proper} learning, where we require the solution to satisfy $\phi \in \Phi$, i.e., the model has to be from the considered model class induced by $\Phi$ and takes the form $f(\cdot,\phi^*)$ for some $\phi^* \in \Omega$, with
$$\frac{1}{D} \sum_{i=1}^D
\ell(f(\hat{x}^i,\phi^*), \hat{y}^i) \leq \min_{\phi \in \Phi} \frac{1}{D} \sum_{i=1}^D
\ell(f(\hat{x}^i,\phi), \hat{y}^i),$$
and this can be relaxed to $\epsilon$-approximate (proper) learning by allowing for an additive error $\epsilon>0$ in the above.
 In contrast, in \emph{improper} learning we allow for a model $g(\cdot)$, that cannot be obtained as $f(\cdot,\phi)$ with $\phi \in \Phi$, satisfying
 $$\frac{1}{D} \sum_{i=1}^D
 \ell(g(\hat{x}^i), \hat{y}^i) \leq \min_{\phi \in \Phi} \frac{1}{D} \sum_{i=1}^D
 \ell(f(\hat{x}^i,\phi), \hat{y}^i),$$
 with a similar approximate version.

\section{Regularized ERM}
\label{sec:regERM}

A common practice to avoid over-fitting is the inclusion of regularizer terms in \eqref{eq:erm}. This leads to problems of the form 
\begin{equation}
  \label{eq:erm-reg}
  \min_{\phi \in \Phi} \frac{1}{D} \sum_{i=1}^D
  \ell(f(\hat{x}^i,\phi), \hat{y}^i) + \lambda\, R(\phi),
\end{equation}
where $R(\cdot)$ is a function, typically a norm, and $\lambda > 0$ is a parameter to control the strength of the regularization. Regularization is generally used to promote generalization and discourage over-fitting of the obtained ERM solution. The reader might notice that our arguments in Section \ref{sec:ERMapprox} regarding the epigraph reformulation of the ERM problem and the tree-decomposition of its intersection graph can be applied as well, since the regularizer term does not add any extra interaction between the data-dependent variables.

The previous analysis extends immediately to the case with regularizers after appropriate modification of the architecture Lipschitz constant $\mathcal{L}$ to include $R(\cdot)$.

\begin{definition}
Consider a regularized ERM problem \eqref{eq:erm-reg} with parameters $D, \Phi, \ell, f, R$, and $\lambda$. We define its \emph{Architecture Lipschitz Constant} $\mathcal{L}(D,\Phi,\ell, f, R, \lambda)$ as
\begin{equation}
    \label{eq:Lipschitz-Arch-regularizers}
    \mathcal{L}(D,\Phi,\ell, f, R, \lambda) \doteq \mathcal{L}_\infty(\ell(f(\cdot , \cdot), \cdot ) + \lambda R(\cdot))
\end{equation}
over the domain $\mathcal{K} = [-1,1]^n \times \Phi \times [-1,1]^m $.
\end{definition}

\section{Binarized Neural Networks}
\label{sec:biu}
A \emph{Binarized activation unit (BiU)} is parametrized
by $p+1$ values $b, a_1, \ldots, a_p$. Upon a binary input vector $z_1, z_2, \ldots, z_p$ the output is binary value $y$ defined by:
\[ y = 1 \ \ \text{if } a^T z > b, \quad \text{and } y = 0 \ \ \text{otherwise}.\]

Now suppose we form a network using BiUs, possibly using different values for the parameter $p$. In terms of the
training problem we have a family of (binary) vectors $x^1, \ldots, x^D$ in $\R^n$ and binary labels
and corresponding binary label vectors $y^1, \ldots, y^D$ in $\R^m$, and as before we want to solve the ERM problem~\eqref{eq:erm}.
Here, the parametrization $\phi$ refers to a choice for the pair $(a,b)$ at each unit.  In the specific case of a network
with $2$ nodes in the first layer and 1 node in the second layer, and $m = 1$, \cite{3nodeblumrivest} showed that
it is NP-hard to train the network so as to obtain zero loss, when $n = D$. Moreover, the authors argued that even if the parameters $(a,b)$ are restricted to be in $\{-1,1\}$, the problem remains NP-Hard. See \cite{courbariaux2016binarized} for an empirically efficient training algorithm for BiUs.

In this section we apply our techniques to the ERM problem \eqref{eq:erm} to obtain an \textit{exact} polynomial-size \emph{data-independent formulation} for each fixed network (but arbitrary $D$) when the parameters $(a,b)$ are restricted to be in $\{-1,1\}$.

We begin by noticing that we can reformulate \eqref{eq:erm} using an epigraph formulation as in \eqref{ERMepigraph}. Moreover, since the data points in a BiU are binary, if we keep the data points as variables, the resulting linear-objective optimization problem is a binary optimization problem as \textbf{BO}. This allows us to claim the following:

\begin{theorem}
Consider a graph $\mathcal{G}$, $p\in \mathbb{N}$ and $D\in \mathbb{N}$. There exists a polytope of size
\[O(2^{p|V(\mathcal{G})|} D),\]
such that any BiU ERM problem of the form \eqref{eq:erm} is equivalent to optimizing a linear function over a face of $P$. Constructing the polytope takes time $O(2^{p|V(\mathcal{G})|} D)$ plus the time required for $O(2^{p|V(\mathcal{G})|})$ evaluations of $f$ and $\ell$.
\begin{proof}
The result follows from applying Theorem \ref{genbtheorem} directly to the epigraph formulation of BiU keeping $x$ and $y$ as variables. In this case an approximation is not necessary. The construction time and the data-independence follow along the same arguments used in the approximate setting before.
\end{proof}
\end{theorem}

The following corollary is immediate.

\begin{corollary}
The ERM problem \eqref{eq:erm} over BiUs can be solved in polynomial time for any $D$, whenever $p$ and the network structure $\mathcal{G}$ are fixed.
\end{corollary}

\end{document}